\documentclass{article}

\usepackage{chngcntr}

\usepackage{microtype}
\usepackage{graphicx}
\usepackage{subcaption}
\usepackage{booktabs} 

\usepackage{hyperref}

\usepackage[preprint]{icml2026}

\usepackage{amsmath}
\usepackage{amssymb}
\usepackage{mathtools}
\usepackage{amsthm}

\usepackage[capitalize,noabbrev]{cleveref}

\theoremstyle{plain}
\newtheorem{theorem}{Theorem}[section]
\newtheorem{proposition}[theorem]{Proposition}

\theoremstyle{definition}

\theoremstyle{remark}

\usepackage[textsize=tiny]{todonotes}

\icmltitlerunning{Cross-Fusion Distance}

\begin{document}

\twocolumn[
  \icmltitle{Cross-Fusion Distance: A Novel Metric for Measuring Fusion and Separability Between Data Groups in Representation Space}
  
  \icmlsetsymbol{equal}{*}

  \begin{icmlauthorlist}
    \icmlauthor{Xiaolong Zhang}{yyy}
    \icmlauthor{Jianwei Zhang}{sch}
    \icmlauthor{Xubo Song}{yyy}
  \end{icmlauthorlist}

  \icmlaffiliation{yyy}{CEDAR, Knight Cancer Institute, Oregon Health and Science University, OR, USA}
  \icmlaffiliation{sch}{Brenden-Colson Center for Pancreatic Care, Oregon Health and Science University, OR, USA}

  \icmlcorrespondingauthor{Xubo Song}{songx@ohsu.edu}

  \icmlkeywords{Distance}

  \vskip 0.3in
]

\printAffiliationsAndNotice{}  

\begin{abstract}
Quantifying degrees of fusion and separability between data groups in representation space is a fundamental problem in representation learning, particularly under domain shift. A meaningful metric should capture \textit{fusion-altering factors} like geometric displacement between representation groups, whose variations change the extent of fusion, while remaining invariant to \textit{fusion-preserving factors} such as global scaling and sampling-induced layout changes, whose variations do not. Existing distributional distance metrics conflate these factors, leading to measures that are not informative of the true extent of fusion between data groups. We introduce \textbf{Cross-Fusion Distance (CFD)}, a principled measure that isolates fusion-altering geometry while remaining robust to fusion-preserving variations, with linear computational complexity. We characterize the invariance and sensitivity properties of CFD theoretically and validate them in controlled synthetic experiments. For practical utility on real-world datasets with domain shift, CFD aligns more closely with downstream generalization degradation than commonly used alternatives. Overall, CFD provides a theoretically grounded and interpretable distance measure for representation learning.

\end{abstract}

\section{Introduction}

In modern representation learning, quantifying the degrees of fusion and separation between groups of latent representations is a fundamental problem, for instance, for robust generalization under domain shift~\cite{gholami2023latent,tamang2025handling}. Learned representations are routinely compared across datasets collected under heterogeneous conditions—such as differing acquisition equipment, environments, or processing pipelines—where nuisance variations can obscure meaningful structure. Effective evaluation and alignment therefore rely on principled and interpretable distance measures that summarize discrepancies between high-dimensional latent distributions~\cite{heiser2020quantitative,horak2021topology,limbeck2024metric}. To this end, a wide range of distributional distances have been adopted. Optimal transport–based measures, most notably the Wasserstein distance (WD), are particularly appealing due to their geometric grounding and ability to yield finite distances even with limited or no support overlap~\cite{torres2021survey}, while kernel-based alternatives such as Maximum Mean Discrepancy (MMD) compare distributions in a chosen reproducing kernel Hilbert space~\cite{wang2021rethinking}. These metrics have been widely used to evaluate representation quality, domain alignment, and domain-shift correction across diverse machine learning applications.



Despite their empirical successes, existing distributional distances suffer from structural limitations, including restrictive distributional assumptions~\cite{borji2022pros}, indirect treatment of geometry via kernel embeddings~\cite{gretton2012kernel,wang2021rethinking}, and the aggregation of conceptually distinct discrepancies into a single scalar value~\cite{santambrogio2015optimal}. In particular, distances such as WD conflate \textbf{fusion-altering factors}—such as inter-group geometric displacement (location shifts) that change fusion—with \textbf{fusion-preserving variations}, including global scaling, sampling-induced layout variations, and dispersion-preserving deformations that leave fusion primarily unchanged~\cite{raghvendra2024new}. This conflation obscures the drivers of observed discrepancies, limiting the interpretability of WD as a faithful measure of fusion between latent representation groups.
Such conflation is particularly problematic when fusion-preserving variations are intrinsic properties of the representation rather than nuisance noise. Structural differences in data sampling, cluster shape, or manifold may encode task-relevant heterogeneity, yet these fusion-preserving variations are penalized in the same manner as fusion-altering shifts, potentially inducing costs to dominate the measured distance and produce large values even when two groups are substantially fused with a minimal separation. Similar pathologies arise in extreme-value metrics such as Hausdorff and Chamfer distances, which are highly sensitive to global scaling and outliers~\cite{chubet2025approximating,lin2024hyperbolic}. Overall, existing distances lack identifiable decomposition between fusion-altering factors and fusion-preserving variations, limiting their diagnostic value.


We introduce a novel distance measure, \textbf{Cross-Fusion Distance (CFD)}, for quantifying the levels of fusion or separation between latent representation groups. CFD is explicitly designed to isolate and quantify \textit{fusion-altering factors} between latent groups while remaining robust to \textit{fusion-preserving variations}. By effectively decoupling these distinct factors, CFD provides a more faithful measure of latent-space fusion under domain shift. This distinction is critical across a wide range of representation-learning settings, including but not limited to biomedical imaging~\cite{Gindra2025ImageAU, kothari2013removing}, where domain shift nuisance frequently occurs across datasets.

This paper makes the following contributions: (1) We identify and formalize a fundamental limitation of existing distributional distances for latent representations: the entanglement of fusion-altering factors with fusion-preserving variations, which undermines their effectiveness and interpretability as indicators of latent fusion and separation under domain shift. (2) We propose \textbf{Cross-Fusion Distance (CFD)}, an interpretable and scale-invariant measure derived from a variance decomposition of the latent point cloud. CFD requires no kernel choices, pairwise distance evaluations, or transport matching, and admits a closed-form computation with linear complexity in sample size and dimensionality. (3) Through theoretical analysis and experiments on synthetic and real-world biomedical datasets, we show that CFD responds monotonically and sensitively to fusion-altering factors while remaining robust to fusion-preserving variations, and exhibits improved calibration and stronger alignment with cross-domain performance degradation compared to existing measures. Together, these results establish CFD as a principled and computationally efficient distance for representation-level analysis under domain shift.

The full Python implementation of CFD is provided in the supplemental material for the review process, and will be released as open-source code on GitHub upon acceptance.


\section{Limitations of Existing Distance Measures}
\label{sec:relatedworks}

\subsection{Displacement–Deformation Conflation}
\label{sec:geometry-topology-coupling}

Many distance measures for comparing latent representation groups lack identifiability due to conflation of fusion-altering factors and fusion-preserving variations. The Wasserstein distance (WD) is a canonical example~\cite{torres2021survey}. Although WD is geometrically well founded and stable under support mismatch, its objective jointly penalizes translation and internal structural deformation, preventing a principled separation of these effects. 

The WD is defined as:


\begin{equation}
\begin{split}
W D_p&(\mu_A, \mu_B)\triangleq \\
& \left(\inf _{\pi \in \Pi(\mu_A, \mu_B)} \int \|z_A - z_B\|^p d \pi(z_A, z_B)\right)^{1/p},
\end{split}
\end{equation}

\noindent where $\mu_A$ and $\mu_B$ are the distributions of latent presentation clusters corresponding to group-A and group-B, respectively; $\pi$ denotes the coupling, and $\Pi$ represents the set of all couplings (i.e., transport plans). The point-to-point distance is defined as $||z_A-z_B||^p$, where $z_A$ and $z_B$ are points sampled from latent representation clusters of group-A and group-B; $p$ denotes the $p^{th}$ moment specifying the norm.

\begin{figure}[ht]
  \begin{center}
    \centerline{\includegraphics[width=0.75\columnwidth]{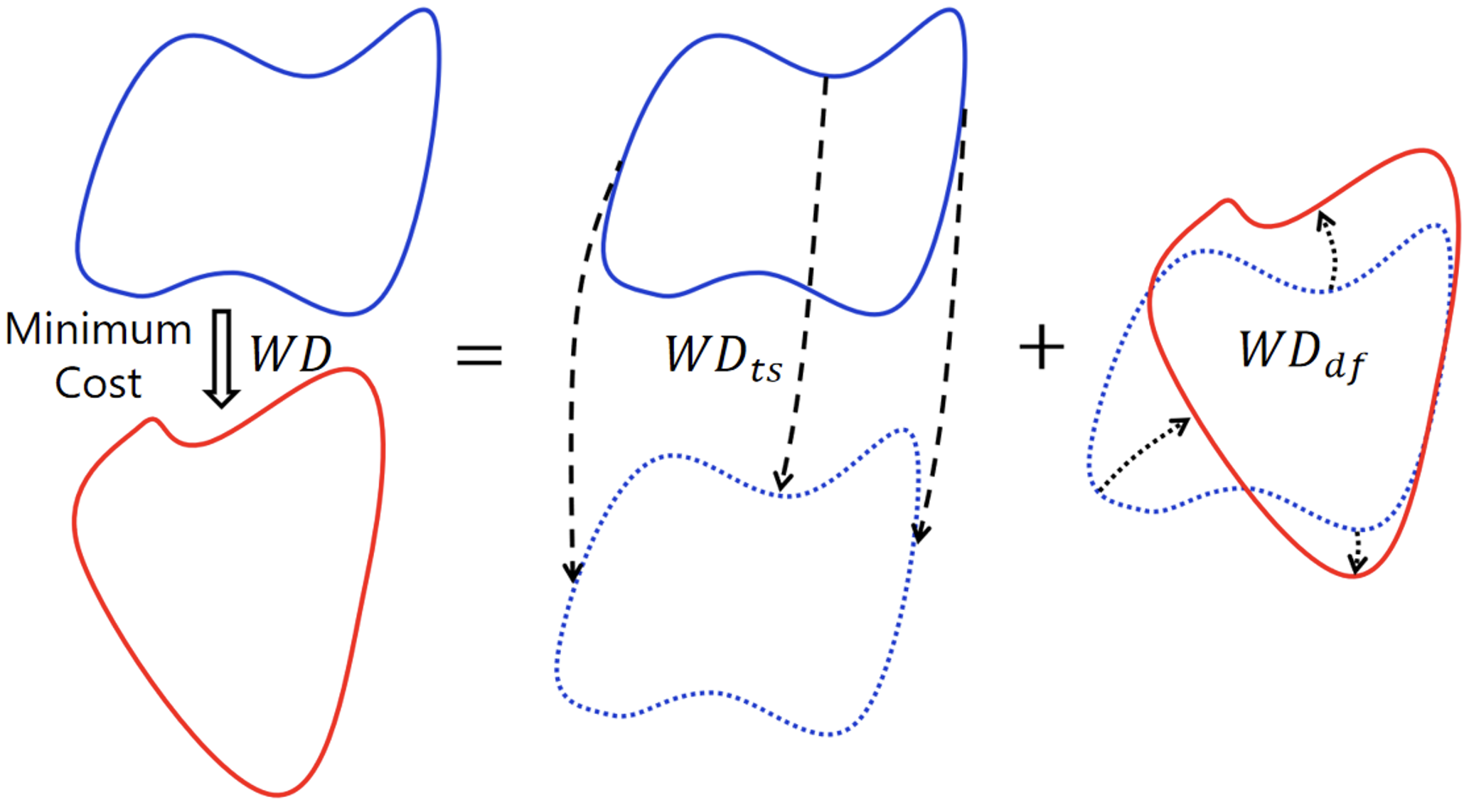}}
    \caption{Wasserstein distance decomposition: $WD_{ts}$ - the translational cost arising from geometric displacement; $WD_{df}$ - the deformation cost due to structural differences.}
    \label{fig:wd-decomp}
  \end{center}
\end{figure}

By definition, the WD represents the minimum cost of transporting one distribution (blue) to another (red), as illustrated in Figure~\ref{fig:wd-decomp}. This transport cost decomposes into a translational component $WD_{ts}$, arising from inter-group geometric displacement (black dashed lines), and a deformation component $WD_{df}$, reflecting structural differences between distributions (black dotted lines). As a result, WD jointly captures both displacement and deformation between latent clusters ~\cite{ambrosio2005gradient, otto2001geometry, takatsu2011wasserstein}:


\begin{equation}
W D = W D_{ts} + W D_{df}.
\end{equation}

Since the topology of latent representation clusters often encodes semantically meaningful structure inherent to the data meanings (e.g., biological variation), the \textit{deformation component} of the WD primarily reflects intrinsic semantics rather than nuisance effects such as domain shift. Consequently, WD conflates inter-group geometric displacement with inherent topological variation. Even when datasets share the same semantic context, intrinsic structural differences can induce substantial deformation costs comparable to translational costs, obscuring interpretation and limiting Wasserstein distance’s ability to isolate separation attributable to domain shift.


\subsection{Scale and Hyperparameter Dependence}

In domain shift settings, representation vector magnitudes may vary due to differences in normalization, encoding algorithms, or preprocessing, making scale-dependent distances (e.g., WD) difficult to compare across experiments. Many commonly used metrics are scale-dependent or rely on sensitive hyperparameter or kernel choices. By construction, distances such as Wasserstein, Hausdorff, and Chamfer are positively homogeneous: uniform rescaling of the latent space proportionally scales the distance value. As a result, variation in feature scale can dominate the measured discrepancy even when the relative fusion between groups is unchanged. Kernel-based measures such as MMD avoid explicit scale dependence but instead introduce kernel dependence, with values governed by kernel choice and bandwidth selection that may vary across experimental settings.


A common mitigation strategy for these metrics might be explicit normalization. However, such normalization is non-canonical and introduces additional design choices: normalization based on extreme statistics can amplify outlier sensitivity, while statistical moment-based normalization may obscure whether observed changes arise from fusion variation or preprocessing conventions. Moreover, normalization and kernel selection add computational overhead when combined with distances that already require quadratic pairwise evaluations or super-quadratic matching ~\cite{santambrogio2015optimal, cuturi2013sinkhorn, lange2023computation, lin2024hyperbolic}.


\subsection{Insensitivity to Level of Fusion}

Another critical yet underappreciated limitation of many distributional distances is their inability to identify dispersion-induced changes in cross-group fusion. When group centroids are fixed, increased within-group dispersion reduces effective separation by enlarging overlap, yet commonly used metrics remain invariant because their objectives depend primarily on absolute geometric or extreme-point distances. Consequently, measures such as Wasserstein, Hausdorff, and Chamfer can assign indistinguishable values to group pairs with substantially different fusion extent, revealing an identifiability failure with respect to dispersion-driven overlap.

\section{Cross-Fusion Distance (CFD)}
\label{sec:methods}


\paragraph{Motivation.}
As discussed in Section~\ref{sec:relatedworks}, many existing distributional measures entangle \emph{fusion-altering factors} with \emph{fusion-preserving variations}. In representation-learning settings where latent topology encodes semantically meaningful variation, incorrectly penalizing it can obscure the interpretation of inter-group fusion or separation. We therefore seek a measure that isolates fusion-altering factors while remaining insensitive to fusion-preserving variations.

\paragraph{Design Principles.} The construction of CFD follows three guiding principles. First, the measure should respond monotonically to fusion-altering factors, such as inter-group geometric displacement and within-group dispersion, whose variations modify the degree of fusion or separation. Second, it should remain stable to fusion-preserving variations, including global scaling, sampling-induced changes in point-cloud layout, and dispersion-preserving deformations or topology changes that do not affect the extent of fusion or separation. Third, the measure should require minimal distributional assumptions. CFD satisfies these criteria by contextualizing fusion contributions of inter-group displacements and within-group dispersions within the total variability of the fused representation space. This results in a scale-invariant ratio that isolates fusion-altering components within the fused representation space and is provably monotonic with respect to fusion-altering variations.

\paragraph{Setup.} Let $z_A = \{z_A^i\}_{i=1}^{n_A}$ and $z_B = \{z_B^j\}_{j=1}^{n_B}$ denote latent representations from groups $A$ and $B$, respectively. Define their empirical centroids as

\begin{equation}
\mu_A = \frac{1}{n_A}\sum_{i=1}^{n_A} z_A^i,
\qquad
\mu_B = \frac{1}{n_B}\sum_{j=1}^{n_B} z_B^j.
\end{equation}

Let $z_{AB} = z_A \cup z_B$ denote the additive union of the two groups, with weights

\begin{equation}
w_A = \frac{n_A}{n_A+n_B},
\qquad
w_B = \frac{n_B}{n_A+n_B},
\end{equation}

and fused centroid

\begin{equation}
\mu_{AB} = w_A \mu_A + w_B \mu_B.
\end{equation}

We define the empirical within-group dispersions as

\begin{equation}
\sigma_A^2 = \left\langle \|z_A - \mu_A\|^2 \right\rangle,
\qquad
\sigma_B^2 = \left\langle \|z_B - \mu_B\|^2 \right\rangle,
\end{equation}

and the dispersion of the fused cloud as

\begin{equation}
\sigma_{AB}^2 = \left\langle \|z_{AB} - \mu_{AB}\|^2 \right\rangle,
\end{equation}

where $\langle \cdot \rangle$ denotes the empirical average.

\paragraph{Variance Decomposition.} The dispersion of the fused cloud admits the following decomposition:

\begin{equation}
\begin{split}
\sigma_{AB}^2 = & w_A \sigma_A^2 + w_B \sigma_B^2 \\
& + w_A \|\mu_A - \mu_{AB}\|^2 + w_B \|\mu_B - \mu_{AB}\|^2.
\end{split}
\label{eq:variance-decomposition}
\end{equation}


The first two terms capture within-group structural variation, while the latter two terms quantify inter-group geometric displacement.

\paragraph{Cross-Fusion Score.} Motivated by Eq.~\eqref{eq:variance-decomposition}, we define the \emph{Cross-Fusion Score (CFS)} as

\begin{equation}
\mathrm{CFS}
=
\frac{w_A \sigma_A^2 + w_B \sigma_B^2}{\sigma_{AB}^2},
\qquad
\mathrm{CFS} \in (0,1].
\end{equation}

CFS measures the proportion of total variance explained by within-group dispersion, with larger values indicating greater overlap between latent groups.

\paragraph{Cross-Fusion Distance.} We define the \emph{Cross-Fusion Distance (CFD)} as

\begin{equation}
\mathrm{CFD} = -\log(\mathrm{CFS}),
\qquad
\mathrm{CFD} \in [0,+\infty).
\end{equation}

This logarithmic transformation maps overlap ratios to a distance scale, yielding a measure that increases monotonically as latent groups become more geometrically separated.

\paragraph{Properties and Interpretation.} From Eq.~\eqref{eq:variance-decomposition}, it follows that

\begin{equation}
\sigma_{AB}^2 \ge w_A \sigma_A^2 + w_B \sigma_B^2,
\end{equation}

with equality if and only if $\mu_A = \mu_B$. Consequently, $\mathrm{CFS} \le 1$ and $\mathrm{CFD} \ge 0$, where $\mathrm{CFD} = 0$ corresponds to perfect overlap of the two latent clouds.

Holding within-group dispersions fixed, CFD is a strictly increasing function of $\|\mu_A - \mu_B\|$, directly reflecting geometric displacement. Formal proofs and additional theoretical results are provided in Appendix~\ref{appendix:CFD}.





\paragraph{Computational Complexity.} The computational complexity of CFD is $O(n \cdot d)$, scaling linearly with the sample size $n$ and latent space dimensionality $d$. This efficiency arises from the closed-form of CFD, which requires only a single pass over the data. In contrast, many alternative distance measures rely on pairwise comparisons or explicit transport plans, resulting in substantially higher computational costs. For example, Sinkhorn-based optimal transport methods typically incur $O(n^2 \cdot d)$ complexity (or worse, depending on regularization and iteration counts), while geometric distances such as the Hausdorff distance similarly require quadratic pairwise evaluations. These theoretical differences are reflected empirically in Fig.~\ref{fig:runtime}, which shows that CFD scales much more favorably with sample size than competing measures. Table~\ref{tab:complexity} summarizes the computational costs of representative distance measures for reference ~\cite{santambrogio2015optimal, cuturi2013sinkhorn, lange2023computation,lin2024hyperbolic}.

\begin{figure}[ht]
  \vskip 0.2in
  \begin{center}
    \centerline{\includegraphics[width=0.75\columnwidth]{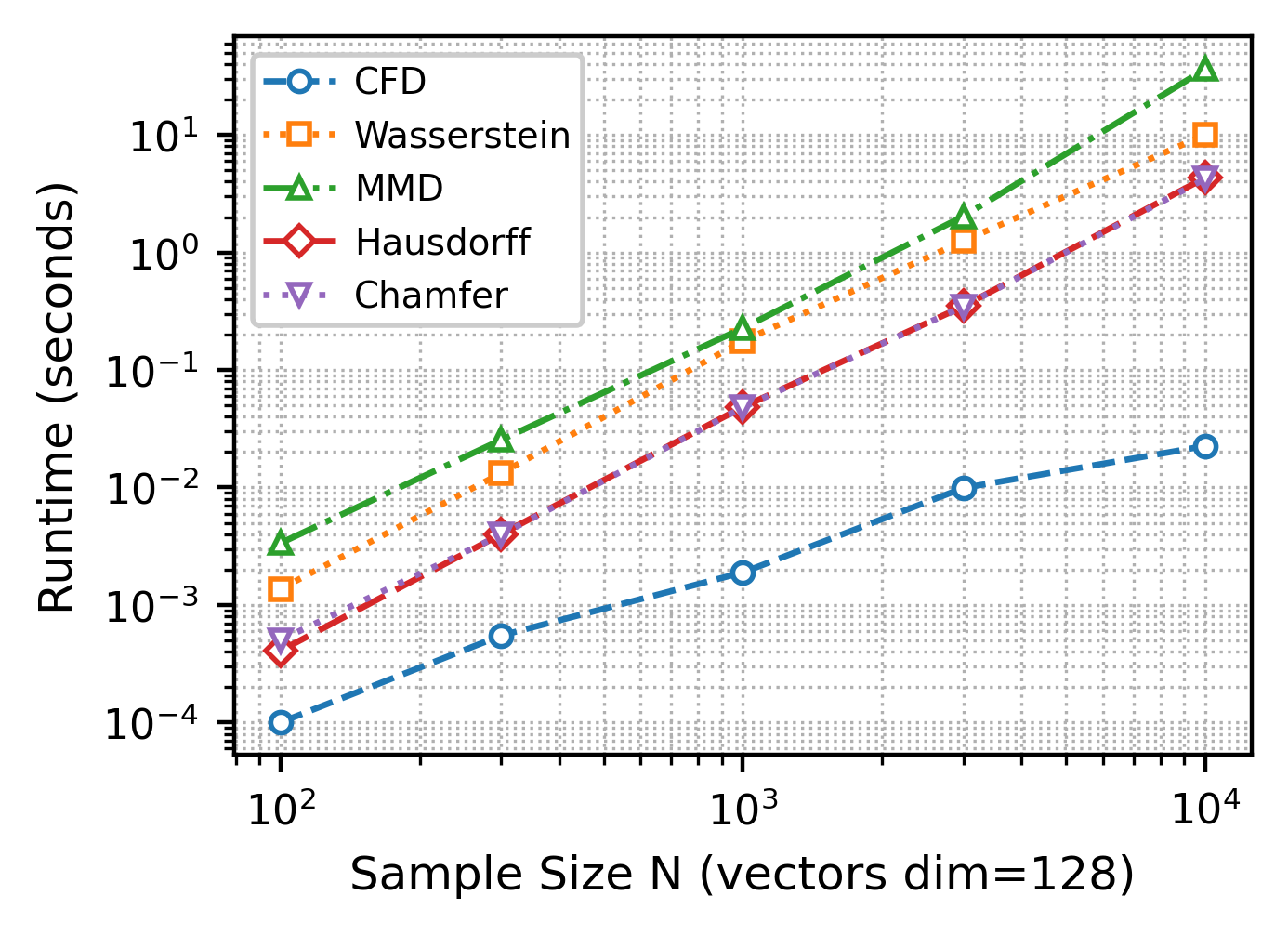}}
    \caption{Empirical runtime comparison of distance measures as a function of sample size.}
    \label{fig:runtime}
  \end{center}
\end{figure}

\begin{table}[t]
\centering
\caption{Computational complexity of different distance measures.}
\label{tab:complexity}
\begin{small}
\begin{tabular}{c c}
\toprule
\textbf{Distance Measures} & \textbf{Complexity} \\
\midrule
Wasserstein2 
& $O(n^3 \log n \cdot d)$ \\
Sinkhorn OT 
& $O(K n^2 \cdot d)$ \\
MMD 
& $O(n^2 \cdot d)$ \\
Hausdorff 
& $O(n^2 \cdot d)$ \\
Chamfer 
& $O(n^2 \cdot d)$ \\
CFD (ours) 
& $\mathbf{O(n \cdot d)}$ \\
\bottomrule
\end{tabular}
\end{small}
\end{table}


\section{Synthetic Experiments}

We evaluate the proposed Cross-Fusion Distance (CFD) on controlled synthetic datasets designed to isolate factors that may affect metric distance quantification, such as inter-group geometric displacement, within-group dispersion, global scaling, and dispersion-preserving topological deformation. By independently manipulating these factors, the experiments assess whether a metric avoids multi-factor conflation, remains stable under scale variation, and responds sensitively to geometric displacement and dispersion-induced changes in cross-group fusion. The results validate the theoretical properties established in Section~\ref{sec:methods} and contrast CFD with commonly used alternatives under well-specified conditions. Implementation details and extended parameter sweeps are provided in Appendix~\ref{app:synthetic}.


\subsection{Experimental Setup and Baseline Measures}

Across all synthetic experiments, we generate latent point clouds $A$ and $B$ in $\mathbb{R}^d$ and evaluate group separation using CFD, Wasserstein distance, MMD, Hausdorff distance, and symmetric Chamfer distance. Results are averaged over $R=50$ Monte Carlo trials with a shared seed schedule for reproducibility. Experiments span latent dimensionalities $d\in\{32,128\}$ and sample sizes $n\in\{300,1000\}$ per group. No feature normalization is applied, as CFD is invariant to isotropic scaling. Wasserstein distance is computed via exact optimal transport for $n=300$ and Sinkhorn approximation ($\varepsilon=0.1$) for $n=1000$. MMD uses an RBF kernel with the median heuristic, and Hausdorff and Chamfer distances follow standard definitions.

Each synthetic experiment is constructed to vary a single factor—geometric displacement, dispersion, scaling, topological deformation, or outlier contamination—while holding all other aspects fixed, enabling unambiguous attribution of metric responses. Unless otherwise stated, representative results for $(d=128, n=300)$ are presented in the main text, with additional dimensionalities and sample sizes reported in Appendix~\ref{app:synthetic}.

\begin{figure*}[ht]
  \vskip 0.2in
  \begin{center}
    \centerline{\includegraphics[width=1.7\columnwidth]{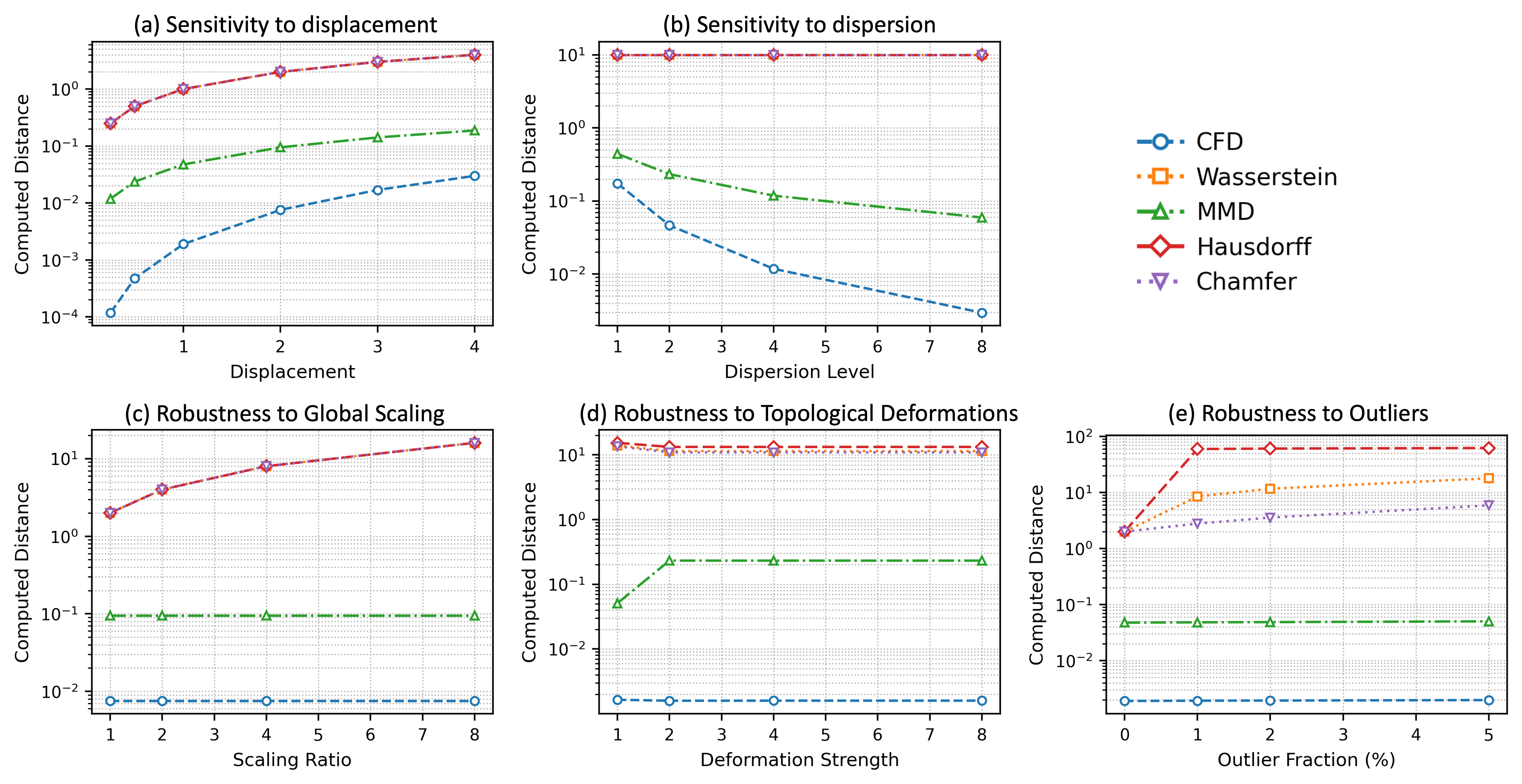}}
    \caption{Synthetic evaluation of distance measures. CFD primarily tracks geometric displacement and dispersion variation while remaining stable to global scaling, topological deformation, and outliers.}
    \label{fig:syn-res}
  \end{center}
\end{figure*}

\subsection{Sensitivity to Geometric Displacement}

We evaluate measure sensitivity to pure geometric displacement in the absence of structural or topological variation. Sensitivity is defined as the measure’s fractional change relative to that of translational displacement~\cite{chitnis2008determining,saltelli2008global,arriola2009sensitivity}. Two groups are sampled from identical Gaussian mixture distributions, with one group translated by $\Delta=\delta e_1$. Varying $\delta$ increases separation while preserving within-group structure.

As shown in Figure~\ref{fig:syn-res}(a), all measures increase monotonically with $\delta$, indicating correct qualitative behavior. By definition of sensitivity, steeper slopes (rather than absolute values) always indicate higher sensitivity. Wasserstein, Hausdorff, and symmetric Chamfer distances exhibit approximately linear growth (whose lines completely overlap in the figure), while MMD responds more weakly due to kernel smoothing. In contrast, CFD increases smoothly with a noticeably steeper slope, reflecting higher sensitivity to geometric displacement. CFD preserves geometric responsiveness while enhancing resolution to genuine spatial separation.

\subsection{Sensitivity to Dispersion Variations}

We then evaluate measure sensitivity to dispersion changes with group centroids held fixed. Increasing dispersion enlarges within-group spread and induces greater clouds' fusion without changing geometric displacement, isolating separability changes driven purely by variance inflation. 

As shown in Figure~\ref{fig:syn-res}(b), Wasserstein, Hausdorff, and symmetric Chamfer distances remain nearly constant across dispersion levels, exhibiting pronounced inertia and markedly weak responsiveness to within-group spread variation. MMD decreases gradually due to limited sensitivity to variance-driven overlap. In contrast, CFD decreases sharply with a substantially steeper slope, capturing fine-grained reductions in separability induced by increased dispersion, implying a superior sensitivity compared to other alternatives.

\subsection{Robustness to Global Scaling}

We evaluate measure robustness to global scaling, where all points are uniformly rescaled by a common factor. Such transformations preserve relative geometry and cloud fusion and therefore should not affect effective separability. As shown in Figure~\ref{fig:syn-res}(c), Wasserstein, Hausdorff, and symmetric Chamfer distances increase substantially with the scaling ratio, as they scale directly with coordinate magnitude despite unchanged relative structure. In contrast, CFD and MMD remain largely stable across scaling levels, reflecting their reliance on relative rather than absolute geometry. CFD maintains near-constant values, correctly withstanding uniform rescaling that does not alter separability.

\subsection{Robustness to Topological Deformation}

We next evaluate robustness to intrinsic topological variation in the absence of geometric displacement. The two groups share an identical centroid, while their internal structure is progressively deformed in a dispersion-preserving manner, altering local shape without changing overall overlap. Such deformations commonly arise in practical applications, particularly in biomedical imaging, where latent representation groups may maintain similar within-group spread across domain shifts while exhibiting substantial structural variation ~\cite{kothari2013removing, yu2023batch, phan2012multiscale, li2020deep}.

As shown in Figure~\ref{fig:syn-res}(d), CFD remains nearly constant across deformation levels, indicating invariance to topology changes that preserve fusion extent. In contrast, Wasserstein, Hausdorff, symmetric Chamfer, and MMD all change noticeably from deformation level 1 to 2, demonstrating susceptibility to intrinsic structural deformation even when centroids and overlap are largely unchanged. These results highlight that existing measures conflate shape variation with separation, whereas CFD discounts fusion-preserving deformations and better isolates effective group separability.

\subsection{Robustness to Outliers}

We evaluate robustness to outliers by replacing a small fraction of points with extreme values while preserving the main mass and overlap of the two clouds. Such perturbations should not materially affect effective separability. 

As shown in Figure~\ref{fig:syn-res}(e), Wasserstein, Hausdorff, and symmetric Chamfer distances increase sharply even for small outlier fractions due to their sensitivity to extreme pointwise distances. In contrast, CFD and MMD remain largely stable with mild slopes, indicating robustness to sparse outliers and a focus on dominant distributional structure rather than isolated extremes.

\subsection{Summary}

Across controlled synthetic settings, CFD consistently isolates effective group separability from confounding factors. It responds sensitively to genuine geometric displacement and dispersion-induced fusion, while remaining invariant to global scaling, dispersion-preserving topological deformation, and sparse outliers. In contrast, commonly used distance measures conflate separation with scale, shape changes, or extreme points, leading to spurious variation. Although MMD performs better than other baselines in several settings, it still shows limited sensitivity and responsiveness to geometric displacement, remains susceptible to topological deformation, depends on kernel and bandwidth selection, and incurs substantially higher computational cost. Overall, these results demonstrate that CFD provides a stable, parameter-free, and scalable metric for quantifying fusion and separability between latent representation groups.
\section{CFD for Biomedical Batch Effect}

Batch effects—systematic, non-biological variations introduced during data acquisition and preprocessing—pose a persistent challenge in biomedical machine learning. When medical images are embedded into latent spaces, such effects manifest as geometric displacements, often causing samples from different batches to occupy weakly overlapping or disjoint regions despite sharing the same underlying biology. Such misalignment degrades decision boundary stability and limits cross-institution generalization.

Effective batch-effect evaluation therefore requires distance measures that are sensitive to inter-batch geometric displacement while remaining robust to within-batch variability. However, most existing distributional distances quantify aggregated discrepancies between latent point clouds without effectively isolating relative displacement from structural deformation, which can lead to overestimation of batch-effect severity or obscure whether latent supports are meaningfully separated.

CFD explicitly isolates geometric displacement while discounting structural variation. We evaluate CFD on real-world histopathology datasets using two criteria: (i) calibration against a known ground-truth baseline, and (ii) alignment with cross-domain performance degradation.

\begin{table*}[t]
\centering
\caption{RDR results across datasets (Camelyon16 and MIDOG21). Columns report RDR for each dataset within the indicated pair.}
\label{tab:rdr-res}
\begin{small}
\begin{tabular}{lcccccccc}
\toprule
 & \multicolumn{2}{c}{\textbf{C16: RAD--UNI}} 
 & \multicolumn{2}{c}{\textbf{MIDOG: ACS--HS}} 
 & \multicolumn{2}{c}{\textbf{MIDOG: ACS--HX}} \\
\cmidrule(lr){2-3} \cmidrule(lr){4-5} \cmidrule(lr){6-7}
\textbf{Measures} 
 & \textbf{RAD} & \textbf{UNI} 
 & \textbf{ACS} & \textbf{HS} 
 & \textbf{ACS} & \textbf{HX} \\
\midrule
CFD           & \textbf{0.0021} & \textbf{0.0019} & \textbf{0.0020} & \textbf{0.0013} & \textbf{0.0030} & \textbf{0.0033} \\
Wasserstein   & 0.5356 & 0.5692 & 0.3467 & 0.2709 & 0.3617 & 0.4543 \\
MMD           & 0.0551 & 0.0529 & 0.0521 & 0.0470 & 0.0626 & 0.0676 \\
Hausdorff     & 0.9208 & 0.9769 & 0.8007 & 0.6975 & 0.7319 & 0.7759 \\
Chamfer       & 0.6236 & 0.6404 & 0.6039 & 0.4628 & 0.4359 & 0.5229 \\
\midrule
Ground Truth  & 0 & 0 & 0 & 0 & 0 & 0 \\
\bottomrule
\end{tabular}
\end{small}
\vskip -0.1in
\end{table*}

\subsection{Ground-Truth Validation on Histopathology Data}

A fundamental challenge in validating batch-effect measures on real-world data is the absence of absolute ground truth: real biomedical datasets rarely can provide a definitive reference for how far apart two domains \emph{should} be in representation space. To overcome this challenge, we construct a ground-truth–anchored evaluation using within-dataset splits as zero-distance references.

\paragraph{Within-dataset splits as zero-distance references.}
For a given dataset, we randomly partition samples from the same batch or dataset into two disjoint subsets. Because these subsets share identical staining procedures, scanner hardware, and preprocessing protocols, they are free of batch effects by construction. Consequently, the true batch-induced distance between the two subsets is zero. Any nonzero distance reported by a distance measure in this setting therefore reflects susceptibility to sampling variability, latent structure deformation, or density fluctuations, rather than the genuine separation between the two subsets.

\paragraph{Relative Distance Ratio (RDR).}
To contextualize within-dataset distances, we normalize them by distances measured across datasets with known batch effects. We define the \emph{Relative Distance Ratio (RDR)} as

\begin{equation}
    \mathrm{RDR}
    =
    \frac{d(A_1 \leftrightarrow A_2)}{d(A \leftrightarrow B)},
\end{equation}

\noindent where $A_1$ and $A_2$ are subsets from the same dataset, and $A$ and $B$ are distinct datasets with clear domain shift. Under the ground truth, $d(A_1 \leftrightarrow A_2) = 0$ and $d(A \leftrightarrow B) = \Delta > 0$, implying $\mathrm{RDR} = 0$. Distance measures that yield closer-to-zero RDR are therefore better aligned with the intended notion of batch-induced separations.

\paragraph{Datasets and protocol.}
We evaluate ground truth anchored behavior on two standard histopathology benchmarks: Camelyon16 and MIDOG21. Camelyon16~\cite{litjens20181399} contains H\&E-stained whole-slide images from Radboud University Medical Center (RAD) and University Medical Center Utrecht (UNI), which differ in staining protocols and scanning conditions. MIDOG21~\cite{aubreville2023mitosis} provides a controlled multi-scanner setting with slides acquired using Aperio ScanScope CS2 (ACS), Hamamatsu S360 (HS), and Hamamatsu XR (HX) scanners. Together, these datasets capture complementary sources of batch effects arising from institutional and hardware variation. Latent representations for all samples are extracted using the pathology foundation model Virchow~\cite{vorontsov2024foundation}.

For each dataset, zero-distance group pairs are constructed by randomly splitting samples from the same acquisition pipeline, and distances between these subsets form the numerator of RDR. The denominator is computed from dataset pairs with known batch effects: RAD–UNI for Camelyon16 and ACS–HS, ACS–HX for MIDOG21. To reduce variance from random partitioning, we repeat this procedure 50 times with independent splits and report RDR values averaged across repetitions.

\textbf{Results.} Table~\ref{tab:rdr-res} reports RDR values for CFD and several widely used distance measures (raw distances are provided in Appendix~C). Across all dataset pairs, CFD consistently produces near-zero RDR values on the order of $10^{-3}$, closely matching the ground-truth expectation and achieving values that are orders of magnitude smaller than all baselines. This indicates substantially superior calibration for batch-effect assessment.



Overall, this ground truth anchored evaluation shows that many commonly used distances are inadequately calibrated for biomedical batch-effect assessment. By isolating batch-induced geometric displacement while remaining robust to within-batch variability, CFD reliably assigns near-zero distance to batch-free splits while preserving sensitivity to true cross-dataset shifts.


\begin{table*}[t]
\centering
\caption{Average representation distances, downstream performance, and cross-domain degradation on MIDOG, averaged over ACS–HS360 and ACS–HXR pairs. The final row reports correlation between each distance and CDDR across normalization methods.}
\label{tab:ccdr-corr}
\begin{small}
\begin{tabular}{rccccccc}
\toprule
\textbf{Normalization} 
& \textbf{CFD} $\downarrow$
& \textbf{Wasserstein} $\downarrow$
& \textbf{MMD} $\downarrow$
& \textbf{Hausdorff} $\downarrow$
& \textbf{Chamfer} $\downarrow$
& \textbf{F1 (Cross)} $\uparrow$
& \textbf{CDDR} $\downarrow$ \\
\midrule
Unnormalized & 0.4282 & 36.9427 & 0.6585 & 37.9457& 22.5876 & 0.369 & 0.3469 \\
Macenko & 0.2243 & 30.9044 & 0.4981 & 23.1042& 35.3897& 0.482 & 0.1204 \\
StainFuser & 0.0459 & 20.2304 & 0.2385& 32.4889& 16.6811 & 0.439 & 0.1459 \\
LMC & 0.0254 & 13.0406 & 0.1787 & 18.6270 & 7.9630 & 0.626 & 0.0726 \\
\midrule
\textbf{Measures Corr.\ w/ CDDR} & \textbf{0.8868} & 0.7418 & 0.8142 & 0.6972& 0.6130 & -- & -- \\
\bottomrule
\end{tabular}
\end{small}
\end{table*}

\subsection{Cross-Domain Performance Degradation}

The practical value of a batch-effect measure ultimately lies in its ability to anticipate real-world generalization failures. In biomedical deployment, practitioners are primarily concerned with how severely performance degrades when models trained under one acquisition condition apply to another. Batch effects are problematic precisely because they induce unpredictable performance loss under domain shift.

\paragraph{Cross-domain degradation measure.}
To assess whether batch-effect measures anticipate real-world generalization failures, we quantify domain-shift–induced performance loss using the \emph{Cross-Domain (or Cross-Dataset) Degradation Rate (CDDR)} ~\cite{mbongo2025one, kalluri2024uda, wang2025grade}. CDDR measures the relative performance drop incurred when a model trained on one dataset is evaluated on another:

\begin{equation}
\mathrm{CDDR}
=
\frac{M_A - M_B}{M_A},
\end{equation}

\noindent where $M_A$ denotes within-dataset performance and $M_B$ denotes cross-dataset performance, measured using standard task metrics such as AUC or F1 score. CDDR isolates the relative impact of domain shift and abstracts away dataset-specific task difficulty and baseline performance variation, thus providing a comparable and interpretable measure of cross-domain generalization and degradation.

\paragraph{Normalization as controlled batch-effect modulation.}
We use histopathology normalization as a controlled intervention to induce systematic variation in domain shift, producing a set of representation spaces with progressively reduced inter-domain separation. Specifically, we consider four normalization settings: \emph{(i) Unnormalized}, which preserves raw batch effects; \emph{(ii) Macenko}~\cite{macenko2009method}, which aligns global stain statistics; \emph{(iii) StainFuser}~\cite{jewsbury2024stainfuser}, which maps images to a shared appearance space via diffusion-based deep learning; and \emph{(iv) LMC}~\cite{anonymous2026lmc}, a normalization method that promotes stain-invariant latent representations through contrastive learning. Together, these settings define a monotonic reduction in batch-induced discrepancy, enabling a direct test of whether representation distances track downstream generalization.


For each normalization setting, we compute the following two values: (i) representation distances using CFD and alternative metrics, and (ii) CDDR from models trained on one dataset and evaluated on another. Results are averaged over the two MIDOG scanner pairs (ACS–HS and ACS–HX). We then compute the correlation between distances and CDDR across normalization settings, assessing how well each measure tracks domain-shift-induced performance degradation.


\paragraph{Results.} Table~\ref{tab:ccdr-corr} summarizes representation distances, downstream cross-domain performance, and degradation rates across normalization settings, averaged over both MIDOG scanner pairs. As normalization strength increases (better normalization method applied), all distance measures exhibit a decreasing trend, reflecting progressively reduced batch-induced separation in latent space. In parallel, downstream cross-domain F1 scores increase monotonically, while CDDR decreases substantially, indicating improved generalization under reduced domain shift.

The bottom row of Table~\ref{tab:ccdr-corr} reports the correlation between each distance measure and CDDR across normalization settings. CFD achieves the highest correlation ($0.8868$), substantially outperforming Wasserstein, MMD, Hausdorff, and Chamfer distances. This result indicates that CFD most accurately captures the aspect of latent separation that governs real-world performance degradation under domain shift. In particular, CFD consistently preserves the correct \emph{ordering} of batch-effect severity across normalization settings in a manner that aligns closely with observed generalization outcomes. This property is critical for deployment-oriented batch-effect assessment, where practitioners must anticipate performance degradation rather than merely quantify representational differences.

Together, these results show that CFD is both well calibrated under ground-truth-anchored tests and predictive of real-world cross-domain performance degradation, establishing a principled link between latent geometry and deployment robustness in biomedical representation learning.

\section{Conclusion}

We introduced \textbf{Cross-Fusion Distance (CFD)}, a principled and interpretable measure for quantifying latent fusion and separability under domain shift by explicitly decoupling fusion-altering factors from fusion-preserving variations. Derived from a variance decomposition of the fused latent cloud, CFD is scale-invariant, computationally efficient, and admits a closed-form evaluation with linear complexity. Synthetic experiments and real-world histopathology benchmarks further showed that CFD is well calibrated and aligns closely with cross-domain performance degradation, outperforming commonly used distance measures. Overall, CFD provides a reliable measure for representation analysis and generalization assessment under domain shift. Future work may extend CFD to multi-group and hierarchical settings, study its behavior under temporal representation drift, and integrate it into representation learning objectives to directly regularize inter-group separation while preserving within-group structures.


\section*{Impact Statements}

This paper presents work whose goal is to advance the field of machine learning and deep learning. There are many potential societal consequences of our work, none of which we feel must be specifically highlighted here.

\bibliography{ref}
\bibliographystyle{icml2026}


\newpage
\appendix
\onecolumn

\renewcommand{\thefigure}{A.\arabic{figure}}
\setcounter{figure}{0}
\renewcommand{\thetable}{A.\arabic{table}}
\setcounter{table}{0}

\section{Theoretical Properties of Cross-Fusion Distance}
\label{appendix:CFD}

In this appendix, we provide formal derivations and theoretical properties of the Cross-Fusion Distance (CFD) introduced in \textcolor{red}{Section~3.1}.

\subsection{Notation}

Let $z_A = \{z_A^i\}_{i=1}^{n_A}$ and $z_B = \{z_B^j\}_{j=1}^{n_B}$ denote latent representations from groups $A$ and $B$, respectively.

Define empirical centroids

\begin{equation}
\mu_A = \frac{1}{n_A} \sum_{i=1}^{n_A} z_A^i,
\qquad
\mu_B = \frac{1}{n_B} \sum_{j=1}^{n_B} z_B^j.
\end{equation}

Let $z_{AB} = z_A \cup z_B$ denote the additive union of the two groups, with total sample size $n = n_A + n_B$ and weights

\begin{equation}
w_A = \frac{n_A}{n},
\qquad
w_B = \frac{n_B}{n}.
\end{equation}

The centroid of the fused cloud is

\begin{equation}
\mu_{AB} = w_A \mu_A + w_B \mu_B.
\end{equation}

Define within-group dispersions

\begin{equation}
\sigma_A^2 = \frac{1}{n_A} \sum_{i=1}^{n_A} \|z_A^i - \mu_A\|^2,
\qquad
\sigma_B^2 = \frac{1}{n_B} \sum_{j=1}^{n_B} \|z_B^j - \mu_B\|^2,
\end{equation}

and fused dispersion

\begin{equation}
\sigma_{AB}^2
=
\frac{1}{n}
\sum_{k=1}^{n}
\|z_{AB}^k - \mu_{AB}\|^2.
\end{equation}

---

\subsection{Variance Decomposition}

\begin{proposition}[Variance Decomposition]
\label{prop:variance-decomposition}

The dispersion of the fused latent cloud admits the decomposition

\begin{equation}
\sigma_{AB}^2
=
w_A \sigma_A^2
+
w_B \sigma_B^2
+
w_A \|\mu_A - \mu_{AB}\|^2
+
w_B \|\mu_B - \mu_{AB}\|^2.
\end{equation}
\end{proposition}

\begin{proof}

We expand the fused dispersion as

\begin{align}
\sigma_{AB}^2
&=
\frac{1}{n}
\left(
\sum_{i=1}^{n_A} \|z_A^i - \mu_{AB}\|^2
+
\sum_{j=1}^{n_B} \|z_B^j - \mu_{AB}\|^2
\right).
\end{align}

Consider the first term. Writing $z_A^i - \mu_{AB} = (z_A^i - \mu_A) + (\mu_A - \mu_{AB})$, we obtain

\begin{align}
\|z_A^i - \mu_{AB}\|^2
&=
\|z_A^i - \mu_A\|^2
+
\|\mu_A - \mu_{AB}\|^2
+
2 (z_A^i - \mu_A)^\top (\mu_A - \mu_{AB}).
\end{align}

Averaging over $i$ eliminates the cross term since $\sum_{i=1}^{n_A} (z_A^i - \mu_A) = 0$.
Thus,

\begin{equation}
\frac{1}{n_A} \sum_{i=1}^{n_A} \|z_A^i - \mu_{AB}\|^2
=
\sigma_A^2 + \|\mu_A - \mu_{AB}\|^2.
\end{equation}

An analogous argument applies to the $B$ group. Substituting both terms and weighting by $w_A$ and $w_B$ yields the stated decomposition.
\end{proof}

---

\subsection{Bounds and Validity of CFD}

\begin{proposition}[Bounds of Cross-Fusion Score]
\label{prop:cfs-bounds}

The Cross-Fusion Score

\[
\mathrm{CFS}
=
\frac{w_A \sigma_A^2 + w_B \sigma_B^2}{\sigma_{AB}^2}
\]

satisfies

\[
0 < \mathrm{CFS} \le 1.
\]
\end{proposition}

\begin{proof}
From Proposition~\ref{prop:variance-decomposition}, the displacement terms $w_A \|\mu_A - \mu_{AB}\|^2$ and $w_B \|\mu_B - \mu_{AB}\|^2$ are nonnegative, implying

\[
\sigma_{AB}^2
\ge
w_A \sigma_A^2 + w_B \sigma_B^2.
\]

Strict positivity follows from $\sigma_{AB}^2 > 0$ unless all points coincide. Equality holds if and only if $\mu_A = \mu_B$.
\end{proof}

---

\subsection{Properties of Cross-Fusion Distance}

\begin{proposition}[Non-negativity and Zero Case]
\label{prop:cfd-nonneg}

The Cross-Fusion Distance

\[
\mathrm{CFD} = -\log(\mathrm{CFS})
\]

satisfies

\[
\mathrm{CFD} \ge 0,
\]

with $\mathrm{CFD} = 0$ if and only if $\mu_A = \mu_B$.
\end{proposition}

\begin{proof}
The result follows directly from Proposition~\ref{prop:cfs-bounds} and the monotonicity of the logarithm.
\end{proof}

---

\begin{proposition}[Monotonicity in Geometric Displacement]
\label{prop:cfd-monotonic}

Holding $\sigma_A^2$ and $\sigma_B^2$ fixed, $\mathrm{CFD}$ is a strictly increasing function of $\|\mu_A - \mu_B\|$.
\end{proposition}

\begin{proof}
From Proposition~\ref{prop:variance-decomposition}, $\sigma_{AB}^2$ increases monotonically with $\|\mu_A - \mu_B\|^2$ through the displacement terms. Since the numerator of $\mathrm{CFS}$ remains constant under fixed within-group dispersions, $\mathrm{CFS}$ decreases monotonically and $\mathrm{CFD}$ increases monotonically with centroid separation.
\end{proof}

---

\begin{proposition}[Limit Cases]
\label{prop:cfd-limits}

The following limits hold:

\begin{itemize}
\item If $\mu_A = \mu_B$, then $\mathrm{CFD} = 0$.
\item If $\|\mu_A - \mu_B\| \to \infty$, then $\mathrm{CFD} \to \infty$.
\end{itemize}

\end{proposition}

\begin{proof}
The first case follows from Proposition~\ref{prop:cfd-nonneg}. For the second case, the displacement terms dominate $\sigma_{AB}^2$, driving $\mathrm{CFS} \to 0$ and hence $\mathrm{CFD} \to \infty$.
\end{proof}

\newpage

\renewcommand{\thefigure}{B.\arabic{figure}}
\setcounter{figure}{0}
\renewcommand{\thetable}{B.\arabic{table}}
\setcounter{table}{0}

\section{Synthetic Experiments}\label{app:synthetic}

This appendix reports a comprehensive suite of synthetic experiments designed to validate the theoretical properties of the proposed Cross-Fusion Distance (CFD) under controlled conditions. In particular, these experiments isolate and manipulate fusion-altering and fusion-preserving variations independently, allowing for a clean empirical assessment of CFD relative to commonly used distributional distance measures.

Unless otherwise stated, all experiments use $R=50$ Monte Carlo repetitions with a shared seed schedule (seed $=42+\text{trial}$). We evaluate across latent dimensionalities $d \in \{32,128\}$ and sample sizes $n \in \{300,1000\}$ per group. No feature normalization is applied, as CFD is defined directly on raw latent representations.

We compare CFD against the following baseline measures: Wasserstein distance (exact optimal transport for $n = 300$, Sinkhorn approximation with regularization parameter $0.1$ for $n = 1000$), Maximum Mean Discrepancy (MMD) with an RBF kernel and median heuristic bandwidth, Hausdorff distance, and symmetric Chamfer distance.

\subsection{Experiments Setup}

\subsubsection{Sensitivity to Geometric Displacement}

We evaluate sensitivity to geometric displacements by sampling two groups from identical distributions with a controlled translational offset. Specifically, we draw $A \sim P$, where $P$ is a $K=4$ Gaussian mixture with equal weights and spherical covariance, and define $B \sim P + \Delta$, where $\Delta = \delta e_1$ shifts all samples along the first coordinate. The displacement magnitude is varied over $\delta \in \{0.25, 0.5, 1, 2, 3, 4\}$. This construction preserves within-group topology while inducing varied inter-group separation.

\subsubsection{Sensitivity to Dispersion Variations}

We evaluate sensitivity to within-group dispersion by sampling two latent groups whose centroids are held at a fixed translational offset  \(\delta = 10.0\) while their spread is systematically varied. Specifically, we draw two groups \(A\) and \(B\) from \(K=4\) Gaussian mixture distributions with equal component weights. Let the mixture component means be \(\{\mu_k\}_{k=1}^K \subset \mathbb{R}^d\). We sample

\[
A \sim P(\sigma), \qquad B \sim P(\sigma)+\delta e_1,
\]
where
\[
P(\sigma) \;=\; \frac{1}{K}\sum_{k=1}^K \mathcal{N}\!\left(\mu_k,\; \sigma^2 I\right),
\]

and \(I\) denotes the identity covariance matrix. 

The dispersion parameter \(\sigma\) is then varied over a predefined grid. As \(\sigma\) increases, each group becomes more dispersed, resulting in a broader spread that induces greater overlap or fusion between the two point clouds. This setup isolates separability changes arising purely from variance inflation rather than geometric displacement. All distance measures are computed on the paired samples \((A,B)\) at each dispersion level, with results averaged over multiple Monte Carlo repetitions to reduce stochastic variability.

\subsubsection{Robustness to Global Scaling}

We evaluate robustness to global scaling by applying a uniform multiplicative transformation to both latent groups. Specifically, we first generate two groups \(A \sim P\) and \(B \sim Q\) in \(\mathbb{R}^d\) following the same procedure as in the displacement experiments, with fixed distributions and a given level of overlap between the groups. We then construct scaled versions of these groups as

\[
A^{(\alpha)} = \alpha A, \qquad B^{(\alpha)} = \alpha B,
\]

where \(\alpha > 0\) is a global scaling factor applied identically to all points in both groups. The scaling ratio \(\alpha\) is varied over a predefined range (e.g., \(\alpha \in \{1, 2, 4, 8\}\)).

By construction, global scaling preserves all relative distances, angular relationships, and overlap proportions between \(A\) and \(B\). Consequently, the effective separability and degree of cloud fusion remain unchanged across scaling ratios. This experiment therefore isolates whether a metric is sensitive to absolute coordinate magnitude. All metrics are computed on the scaled pairs \((A^{(\alpha)}, B^{(\alpha)})\) for each \(\alpha\), and results are averaged over repeated Monte Carlo trials to mitigate sampling variability.

\subsubsection{Robustness to Topological Deformation}

We evaluate robustness to topological deformation by modifying the internal structure of the latent groups while explicitly holding their centroids fixed. Specifically, we construct two groups \(A\) and \(B\) in \(\mathbb{R}^d\) such that

\[
\mu_A = \mu_B = \mathbf{0},
\]

ensuring identical centroids by design. Group A is set with a fixed unimodal internal structure and group B introduces further deformations by progressively varying the number of mixture components of Gaussian mixture models.

Formally, for a given deformation level \(K\), we define

\[
B \sim P_K,
\]

where

\[
P_K \;=\; \frac{1}{K}\sum_{k=1}^K \mathcal{N}\!\left(\mu_k,\; \sigma^2 I\right),
\]

and the component means \(\{\mu_k\}_{k=1}^K\) are arranged symmetrically in \(\mathbb{R}^d\) such that \(\sum_{k=1}^K \mu_k = \mathbf{0}\). This symmetric construction guarantees that increasing the modality count alters local shape and internal structure while preserving the global centroid and overall dispersion. The deformation strength is controlled by varied number of mixture components over \(K \in \{1,2,4,8\}\), with \(\sigma\) held fixed across all settings.

As \(K\) increases, the latent distribution transitions from unimodal to increasingly multi-modal, introducing intrinsic topological variation without materially changing the overall extent of overlap between the two groups. This setup therefore isolates whether a metric responds robustly to structural deformation alone, rather than to entangled changes with inter-group geometric displacements or group dispersion variations. All metrics are evaluated on the paired samples \((A,B)\) for each value of \(K\), with results averaged over multiple Monte Carlo realizations to reduce sampling variability.

\subsubsection{Robustness to Outliers}

We evaluate robustness to outliers by contaminating each latent group with a small fraction of extreme observations while preserving the main mass, centroid structure, and effective overlap of the two clouds. This experiment is designed to isolate whether a distance measure disproportionately reacts to sparse, heavy-tailed noise that does not materially alter cloud fusion.

Specifically, we revisit the geometric displacement setting with a fixed translational offset \(\delta = 2.0\). Let the uncontaminated groups be generated as

\[
A \sim P, \qquad B \sim P + \delta e_1,
\]

where \(P\) is a \(K=4\) Gaussian mixture with equal weights and spherical covariance, as defined previously. We then introduce outliers by replacing an \(\epsilon\%\) fraction of points in each group with samples drawn independently from a wide Gaussian distribution,

\[
\mathcal{N}(\mathbf{0}, \tau^2 I),
\]

with scale parameter fixed at \(\tau = 5.0\). The contamination level is varied over

\[
\epsilon \in \{0, 1, 2, 5\}.
\]

By construction, the majority of samples in both \(A\) and \(B\) continue to follow the original displaced mixture distribution, and the degree of overlap between the two primary clouds remains largely unchanged. Increasing \(\epsilon\) therefore introduces progressively heavier tails without modifying the dominant geometry or fusion of the point clouds. All metrics are computed on the contaminated pairs \((A_\epsilon, B_\epsilon)\) for each outlier fraction, with results averaged over repeated Monte Carlo draws to reduce sampling variability.

\subsection{Additional Synthetic Results Across Dimensions and Sample Sizes}

To assess the robustness of our findings beyond the main experimental configuration, we report additional synthetic evaluations under varying latent dimensionalities and sample sizes. The following figures present results for \(d \in \{32,128\}\) and \(n \in \{300,1000\}\), covering a broad range of regimes commonly encountered in representation learning. Across all settings, the qualitative behavior of the evaluated metrics remains consistent with the results represented in Section 4. These additional experiments confirm that the advantages of CFD are not artifacts of a specific dimensionality or sample size, but persist across substantially different latent-space regimes.

\begin{figure*}[ht]
  \vskip 0.2in
  \begin{center}
    \centerline{\includegraphics[width=\columnwidth]{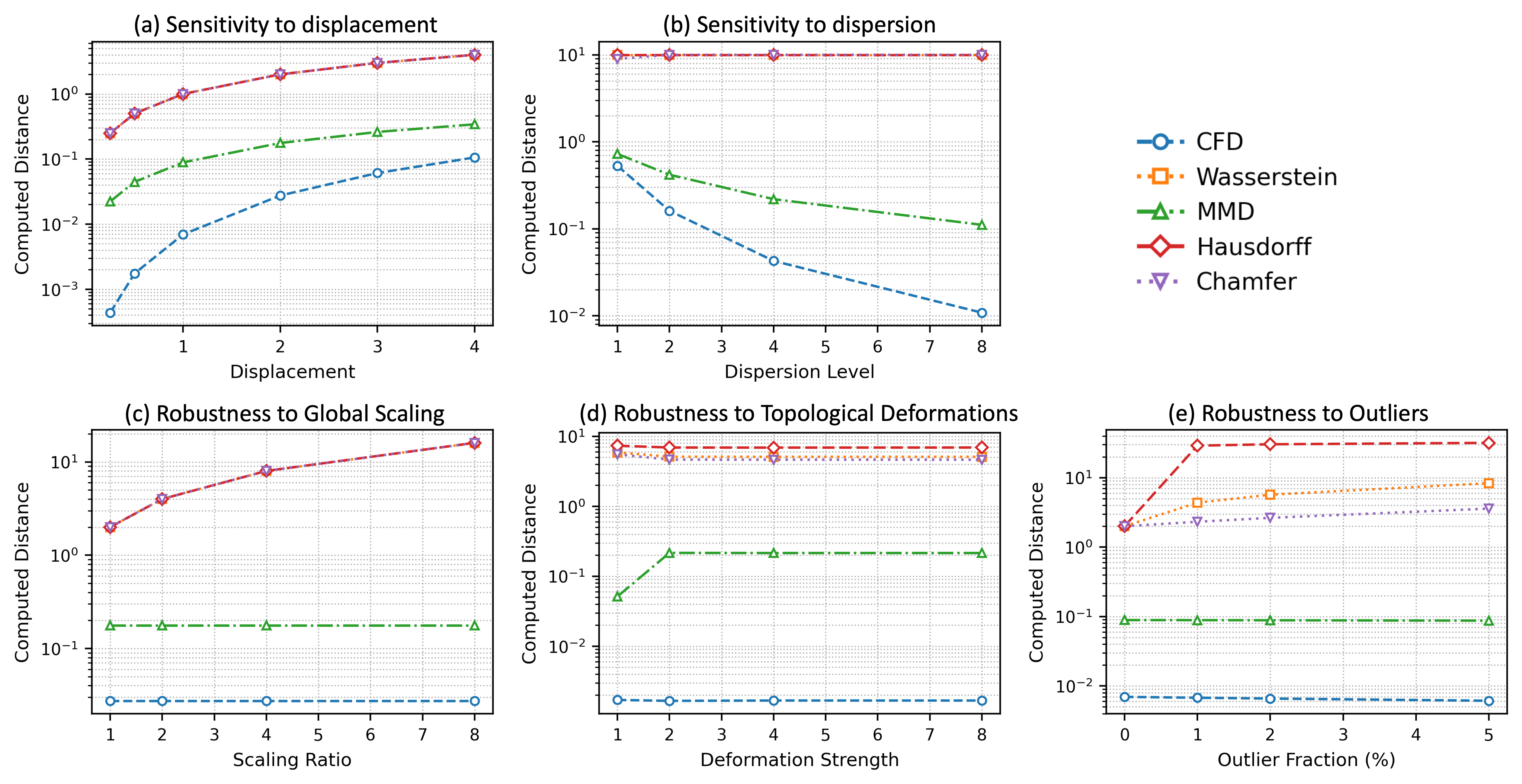}}
    \caption{Synthetic evaluation of distance measures for latent dimensionalities $d =32$ and sample sizes $n=300$.}
  \end{center}
\end{figure*}

\begin{figure*}[ht]
  \vskip 0.2in
  \begin{center}
    \centerline{\includegraphics[width=\columnwidth]{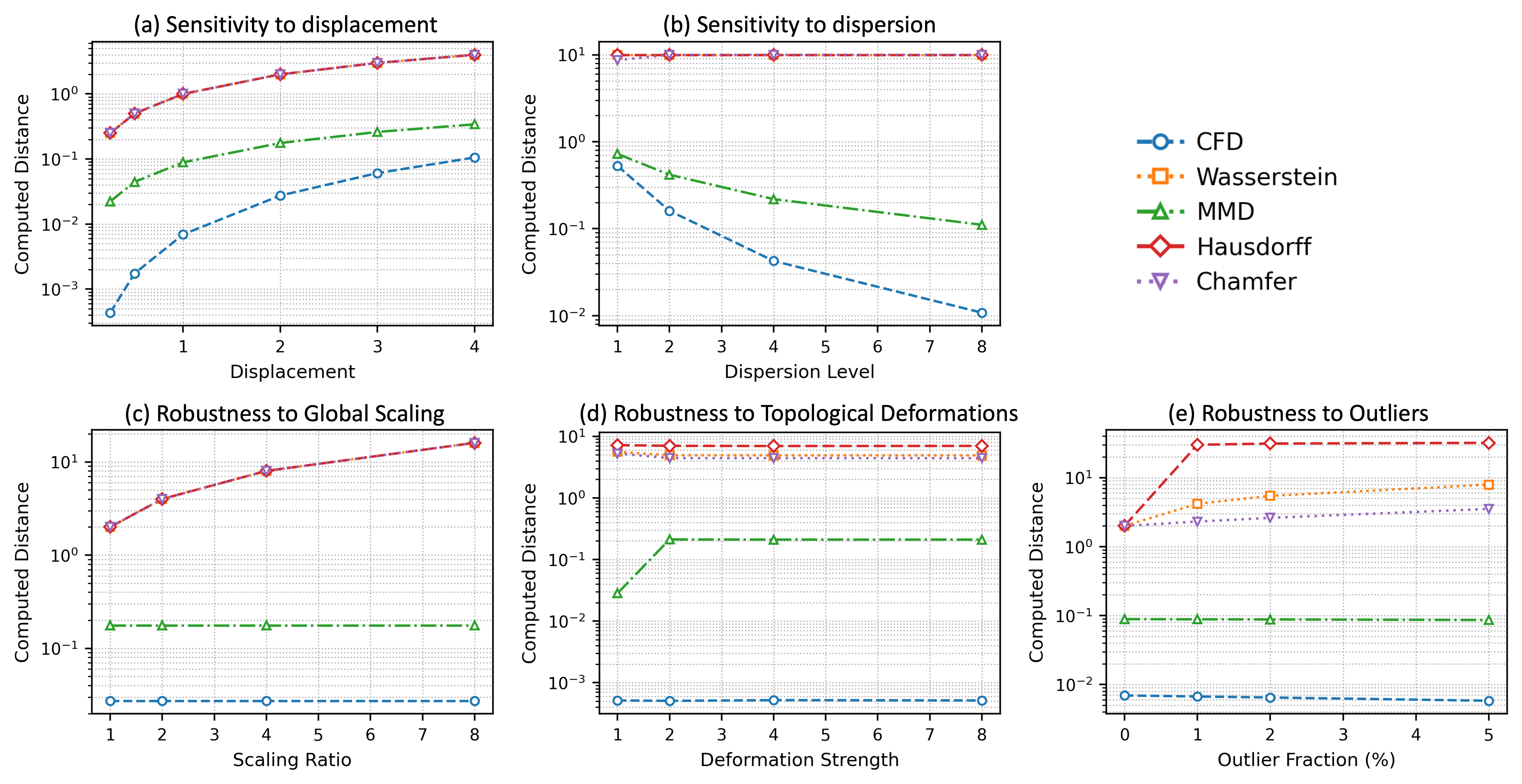}}
    \caption{Synthetic evaluation of distance measures for latent dimensionalities $d =32$ and sample sizes $n=1000$.}
  \end{center}
\end{figure*}

\begin{figure*}[ht]
  \vskip 0.2in
  \begin{center}
    \centerline{\includegraphics[width=\columnwidth]{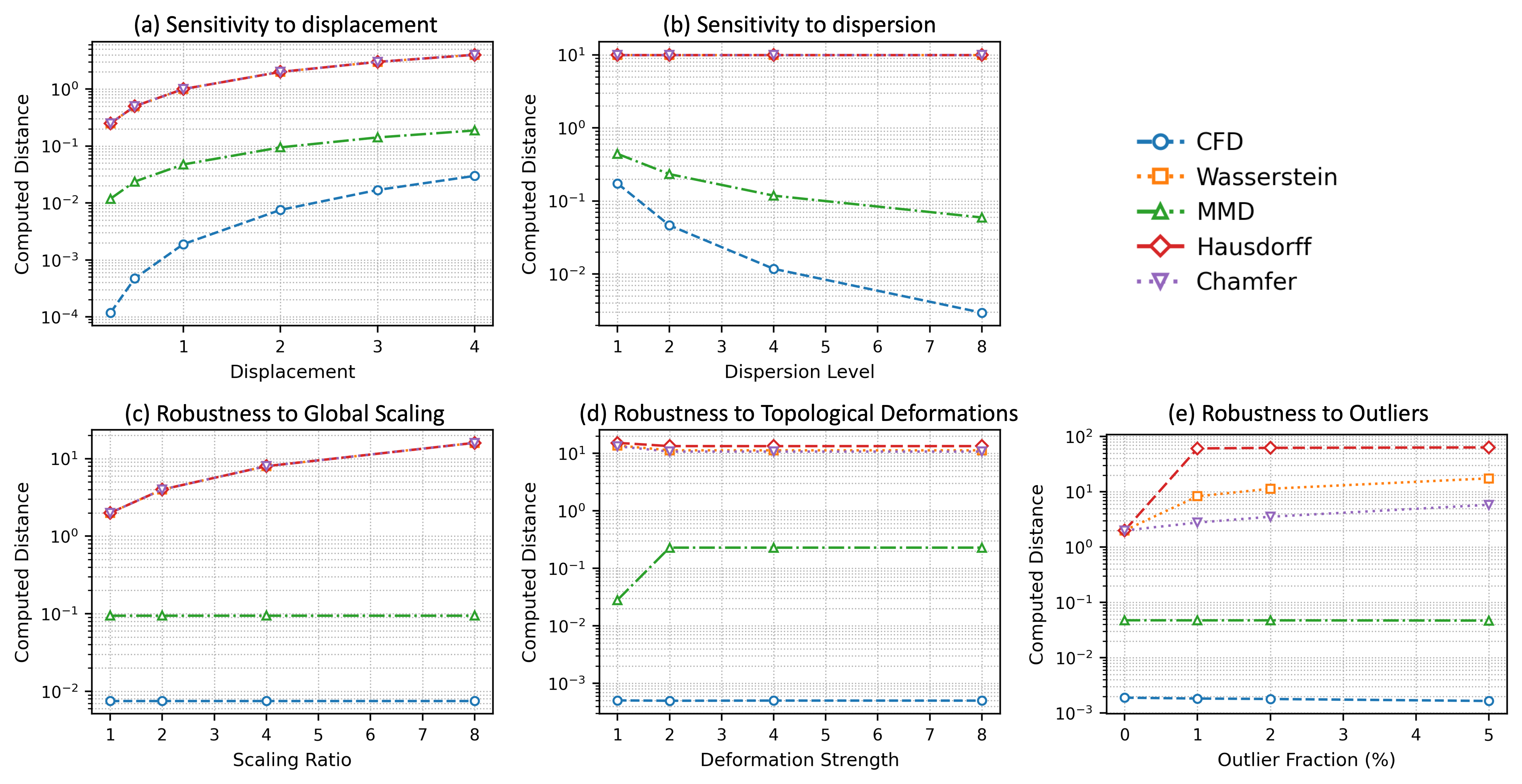}}
    \caption{Synthetic evaluation of distance measures for latent dimensionalities $d =128$ and sample sizes $n=1000$.}
  \end{center}
\end{figure*}

\renewcommand{\thefigure}{C.\arabic{figure}}
\setcounter{figure}{0}
\renewcommand{\thetable}{C.\arabic{table}}
\setcounter{table}{0}

\section{Full Results for Ground-Truth–Anchored Validation on Real Histopathology Data}
\label{appendix:full-ground-truth}

This appendix reports the complete distance values underlying the Relative Distance Ratio (RDR) analysis presented in the main text. While the main paper focuses on normalized RDR scores to facilitate cross-metric comparison, the raw distances reported here provide transparency into the absolute scale and variability of each metric under both no-shift and batch-shift conditions.

For each dataset, we separately report (i) distances between two randomly sampled subsets drawn from the same acquisition pipeline, which serve as ground-truth zero-distance references, and (ii) distances between datasets acquired under distinct staining protocols or scanner hardware, which constitute the true batch-effect conditions used as RDR denominators. All values are reported as mean $\pm$ standard deviation over repeated random splits.

These full results confirm that differences in RDR across metrics are not driven by trivial scale effects or noise amplification. Instead, they reflect systematic differences in how metrics respond to within-domain variability versus inter-domain displacement. In particular, metrics that assign substantial nonzero distance to within-dataset splits inevitably yield inflated RDR values, even when cross-dataset distances are large. By contrast, CFD maintains near-zero distances under the ground-truth condition while remaining sensitive to true batch shifts, leading to consistently well-calibrated RDR scores.

\begin{table}[htbp]
\centering
\caption{Within-dataset distances on Camelyon16 computed between two randomly sampled subsets drawn from the same acquisition source. These pairs constitute ground-truth zero-distance references for RDR evaluation. Values are reported as mean $\pm$ standard deviation over repeated splits.}
\begin{tabular}{lcc}
\toprule
\textbf{Distance Metric} & \textbf{2 Subsets from RAD} & \textbf{2 Subsets from UNI} \\
\midrule
CFD        & $0.00027 \pm 0.00005$ & $0.00025 \pm 0.00006$ \\
W2         & $16.8032 \pm 0.16108$    & $17.8574 \pm 0.14284$   \\
MMD        & $0.02011 \pm 0.00281$   & $0.01929 \pm 0.00193$  \\
Hausdorff & $40.8993 \pm 1.84970$    & $43.38930 \pm 2.25711$    \\
Chamfer   & $14.8684 \pm 0.05756$    & $15.2671 \pm 0.09466$   \\
\midrule
\textbf{Ground Truth} & \textbf{0} & \textbf{0} \\
\bottomrule
\end{tabular}
\end{table}

\begin{table}[htbp]
\centering
\caption{Cross-dataset distances between RAD and UNI on Camelyon16, representing true batch-shift conditions used as the RDR denominator. Values are reported as mean $\pm$ standard deviation over repeated trials.}
\begin{tabular}{lccccc}
\toprule
\textbf{Dataset Pair} & \textbf{CFD} & \textbf{W2} & \textbf{MMD} & \textbf{Hausdorff} & \textbf{Chamfer} \\
\midrule
RAD $\leftrightarrow$ UNI
& $0.1295 \pm 0.0022$
& $31.3701 \pm 0.0679$
& $0.3649 \pm 0.0031$
& $44.4162 \pm 0.7032$
& $23.8412 \pm 0.1632$ \\
\bottomrule
\end{tabular}
\end{table}

\begin{table}[htbp]
\centering
\caption{Within-dataset distances on MIDOG21 computed between two randomly sampled subsets from the same scanner type (ACS, HS360, or HXR). These pairs serve as ground-truth zero-distance references for RDR evaluation. Mean $\pm$ standard deviation over repeated splits.}
\begin{tabular}{lccc}
\toprule
\textbf{Distance Metric} 
& \textbf{2 Subsets from ACS} 
& \textbf{2 Subsets from HS360} 
& \textbf{2 Subsets from HXR} \\
\midrule
CFD 
& $0.00101 \pm 0.00028$ 
& $0.00067 \pm 0.00016$ 
& $0.00112 \pm 0.00033$ \\
W2 
& $13.0811 \pm 0.10889$ 
& $10.2180 \pm 0.08144$ 
& $16.4262 \pm 0.19541$ \\
MMD 
& $0.03747 \pm 0.00917$ 
& $0.03377 \pm 0.00453$ 
& $0.04046 \pm 0.00772$ \\
Hausdorff 
& $29.0148 \pm 1.59810$ 
& $25.2751 \pm 1.89270$ 
& $30.7773 \pm 1.12900$ \\
Chamfer 
& $11.4366 \pm 0.03994$ 
& $8.76662 \pm 0.02203$ 
& $13.7220 \pm 0.04921$ \\
\midrule
\textbf{Ground Truth} 
& \textbf{0} 
& \textbf{0} 
& \textbf{0} \\
\bottomrule
\end{tabular}
\end{table}

\begin{table}[htbp]
\centering
\caption{Cross-scanner distances on MIDOG21 between ACS--HS360 and ACS--HXR, representing true batch-effect conditions used as the RDR denominator. Mean $\pm$ standard deviation over repeated trials.}
\begin{tabular}{lccccc}
\toprule
\textbf{Dataset Pair} 
& \textbf{CFD} 
& \textbf{W2} 
& \textbf{MMD} 
& \textbf{Hausdorff} 
& \textbf{Chamfer} \\
\midrule
ACS $\leftrightarrow$ HS360
& $0.517 \pm 0.003$
& $37.721 \pm 0.031$
& $0.719 \pm 0.001$
& $36.242 \pm 0.456$
& $18.938 \pm 0.069$ \\
ACS $\leftrightarrow$ HXR
& $0.340 \pm 0.003$
& $36.164 \pm 0.100$
& $0.598 \pm 0.002$
& $39.649 \pm 0.209$
& $26.238 \pm 0.064$ \\
\bottomrule
\end{tabular}
\end{table}


\end{document}